\documentclass{article}

\usepackage{arxiv}

\usepackage[utf8]{inputenc}
\usepackage[T1]{fontenc}

\usepackage{amsmath,amsfonts,bm}

\def\eqref#1{equation~\ref{#1}}

\def\1{\bm{1}}

\DeclareMathAlphabet{\mathsfit}{\encodingdefault}{\sfdefault}{m}{sl}
\SetMathAlphabet{\mathsfit}{bold}{\encodingdefault}{\sfdefault}{bx}{n}

\usepackage{amssymb}
\usepackage{amsthm}
\usepackage{booktabs}
\usepackage{graphicx}
\usepackage{float}
\usepackage{microtype}
\usepackage{placeins}
\usepackage[round,authoryear]{natbib}
\usepackage{xcolor}
\usepackage{url}
\usepackage{hyperref}

\title{Beyond Quadratic Loss: The Stability Phase Diagram of Adam}
\author{
  Gaoxiang Tang \\
  IIIS, Tsinghua University \\
  \And
  Huanran Chen \\
  College AI, Tsinghua University\\
  \And
  Ziming Liu \\
  College AI, Tsinghua University \\
  Shanghai Qizhi Institute \\
  MetaCircle \\
}
\date{}

\renewcommand{\headeright}{Preprint}
\renewcommand{\undertitle}{Preprint}
\renewcommand{\shorttitle}{Beyond Quadratic Loss: The Stability Phase Diagram of Adam}

\hypersetup{
  colorlinks=true,
  linkcolor=black,
  citecolor=blue!45!black,
  urlcolor=blue!55!black,
  pdftitle={Beyond Quadratic Loss: The Stability Phase Diagram of Adam},
  pdfauthor={Gaoxiang Tang and Huanran Chen and Ziming Liu},
  pdfsubject={Machine Learning and Optimization},
}

\newtheorem{theorem}{Theorem}
\newtheorem{lemma}[theorem]{Lemma}
\newtheorem{corollary}[theorem]{Corollary}

\theoremstyle{definition}
\newtheorem{definition}[theorem]{Definition}
\theoremstyle{plain}

\newenvironment{experimentsetup}
  {\par\begingroup\setlength{\parskip}{0pt}}
  {\par\endgroup}

\begin{document}

\maketitle

\begin{abstract}
Loss spikes are recurrent instabilities in neural-network training and can arise from multiple mechanisms. For Adam in particular, macroscopic loss spikes have been linked to optimizer dynamics, yet how its two momentum timescales govern them remains unclear. We investigate this dependence by mapping training dynamics across the \((\beta_1,\beta_2)\) plane. Across a range of model--task settings, an approximately linear boundary, \(1-\beta_2=C(1-\beta_1)\), separates spiky from non-spiky dynamics, whereas a one-dimensional quadratic loss produces approximately cubic slope. A one-dimensional superquadratic loss \(L(x)\propto|x|^n\) recovers the near-linear scaling and links the boundary coefficient to the effective loss exponent \(n\). We further show that confident cross-entropy losses develop a core--wall landscape comprising a narrow quadratic core followed by a steep wall, which produces effective superquadratic behavior at the scale of an optimizer update. Together, these results connect Adam loss spikes to both the mismatch between momentum timescales and finite-scale superquadratic loss geometry beyond the Hessian.
\end{abstract}

\section{Introduction}\label{introduction}

\begin{figure}[H]
\centering
\includegraphics[width=\linewidth,keepaspectratio]{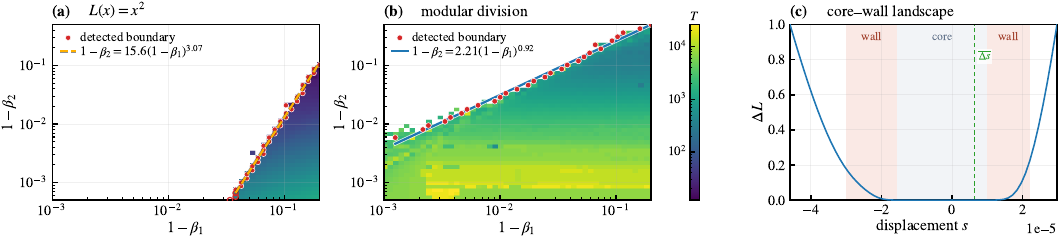}
\caption{Overview of the empirical boundary scaling and finite-scale loss geometry. (a) Adam stability phase diagram in the \((\beta_1,\beta_2)\) plane for \(L(x)=x^2/2\). A power-law fit gives \(1-\beta_2\simeq15.6(1-\beta_1)^{3.07}\), showing the approximately cubic scaling of the quadratic model. (b) Phase diagram for the modular-division Transformer, whose detected boundary follows the near-linear relation \(1-\beta_2\simeq2.21(1-\beta_1)^{0.92}\). In (a,b), color denotes the measured oscillation period. (c) Directional loss slice \(L(s)=L(\theta+s\hat d_t)\) at a pre-spike point of the modular-division Transformer, where $\hat d_t$ is the Adam-preconditioned gradient direction. The green dashed line shows the magnitude of a typical update, which is comparable to the flat-core width. }
\label{fig:overview}
\end{figure}

Loss spikes are abrupt, macroscopic loss excursions that recur across neural-network training settings. Their causes depend on the training regime. In Adam training, such spikes have been widely reported \cite{chowdhery2023palm,molybog2023adam,thilak2022slingshot}. Under a local quadratic approximation, gradient descent is stable when \(\lambda_{\max}(H_t)<2/\eta\). Crossing this threshold can induce Edge-of-Stability dynamics characterized by non-monotonic loss excursions and spikes \cite{xing2018walk,jastrzebski2020breakeven,cohen2021gradient}. For adaptive optimizers, the analogous criterion replaces the raw Hessian with a preconditioned Hessian, leading to the Adaptive Edge of Stability for Adam and RMSProp \cite{cohen2023adaptive,cohen2025centralflows}. Sustained violations of this effective stability threshold have also been linked to macroscopic Adam loss spikes \cite{bai2026adaptive}. Other mechanisms apply in more specific settings. Numerical Feature Inflation explains spikes in low-precision models trained with cross-entropy loss \cite{liu2026grokking}, while weight-norm criticality applies to models with scale-invariant layers \cite{li2026weightnorm}.

Adam's first and second gradient moments have memory timescales set by \(\beta_1\) and \(\beta_2\), respectively \cite{kingma2014adam}.
\textbf{We investigate how the mismatch between \(\beta_1\) and \(\beta_2\) governs loss spikes.}
We focus on macroscopic loss spikes triggered by violations of Adam's EoS condition, rather than the microscopic oscillations associated with short-period attractors \cite{bock2019nonconvergence,fong2026provable}. 
Across diverse models and tasks, we identify the near-linear stability boundary \(1-\beta_2\propto(1-\beta_1)\) shown in Figure~\ref{fig:overview}(b). 
At fixed \(\beta_1\), the spiky phase lies on the larger-\(\beta_2\) side of this boundary.
To isolate the underlying dynamics, we study Adam on the one-dimensional quadratic loss \(L(x)=kx^2/2\). This model instead produces the approximately cubic boundary \(1-\beta_2\propto(1-\beta_1)^3\) in Figure~\ref{fig:overview}(a). 
\citet{bai2026degenerate} analyze the \((\beta_1,\beta_2)\) phase transition on even-degree degenerate polynomials, obtaining a linear boundary but a phase ordering different from that observed in real models.
These discrepancies motivate an analysis beyond the conventional infinitesimal quadratic approximation.

We then analyze the Adam stability phases for a superquadratic one-dimensional loss \(L(x)=|x|^n\) with \(n>2\). We find a wedge-shaped spiky phase bounded approximately by
\begin{equation}
C_L(1-\beta_1)\geq 1-\beta_2\geq C_R(1-\beta_1).
\end{equation}
The right boundary follows from the competition between the effective learning rate \(\eta_{\rm eff}\) and the critical learning rate \(\eta_{\rm crit}\), giving \(C_R=2(n-2)/n\). The left boundary involves more complex oscillatory dynamics and is described empirically by \(C_L\simeq12(n-2)/(3n-4)\).
The right boundary is consistent with \citet{bai2026degenerate}, but does not appear in real models because the second derivative at the minimum is nonzero. The boundary we observe in real models instead corresponds to the left boundary.

Local loss landscapes are often analyzed in the infinitesimal limit through a second-order Taylor expansion. In neural networks, the corresponding Hessian spectra typically contain a broad near-zero bulk and a few isolated outliers \cite{sagun2018hessian}. This structure has motivated river-valley descriptions of neural-network optimization \cite{xing2018walk,wen2024river}.

We instead analyze directional loss landscapes along the preconditioned gradient at finite scales and find a \textbf{core--wall landscape} composed of a flat quadratic core, a steep superquadratic wall, and an outer rollover, as illustrated in Figure~\ref{fig:overview}(c).
\textbf{We show that this structure arises naturally from confident cross-entropy and derive the scale of its quadratic core.} During training, the landscape undergoes finite-scale progressive sharpening, in which the preconditioned sharpness increases while the flat core contracts. Once a typical update becomes comparable to the core width, Adam probes the wall and encounters a large effective exponent \(n\). We estimate this exponent at pre-spike points and substitute it into the empirical relation for \(C_L\). The resulting estimates capture the scale of the observed boundary coefficients.

This work makes three contributions. First, we identify a near-linear phase boundary separating spiky and non-spiky Adam dynamics across diverse models and tasks. Second, we study Adam stability in the context of superquadratic losses and show how the memory timescales of the first and second moments set the phase boundaries. Third, we study the local loss landscape at finite scales, show that its core--wall structure arises naturally from confident cross-entropy, and relate its effective exponent at the optimizer's update scale to the empirical boundary coefficient.

\section{Stability Phase Boundaries Across Six Models}\label{sec:beta-boundaries-and-slopes-for-six-models}

We study six model--task settings: one-layer Transformers \cite{vaswani2017attention} for modular division and addition modulo 53, a character-level Transformer for next-character prediction on Tiny Shakespeare, an MLP and a VGG11-style CNN \cite{simonyan2015verydeep} for CIFAR-10 classification \cite{krizhevsky2009learning}, and a convolutional autoencoder for binarized MNIST reconstruction \cite{lecun1998gradient}. All models use AdamW \cite{loshchilov2019decoupled} with cross-entropy for classification and language modeling or binary cross-entropy for reconstruction. Within each setting, only \(\beta_1\) and \(\beta_2\) vary, while the architecture, data, random seed, and other optimizer hyperparameters remain fixed. Appendix~\ref{experimental-setup-for-the-six-models} gives the full configurations.

\begin{figure}[t]
\centering
\includegraphics[width=\linewidth,keepaspectratio]{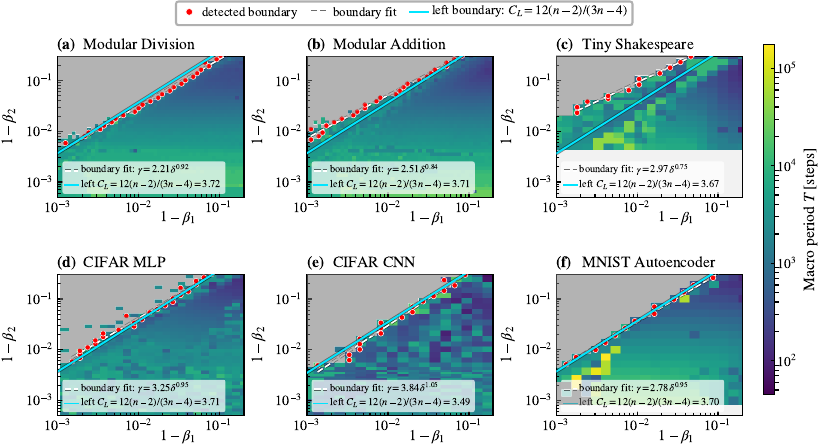}
\caption{Phase boundaries for six models. (a--f) Modular division, modular addition, Tiny Shakespeare, a CIFAR-10 MLP, a CIFAR-10 CNN, and an MNIST autoencoder, respectively. The axes are the memory gaps \(\delta=1-\beta_1\) and \(\gamma=1-\beta_2\), and color indicates the detected macroscopic oscillation period in optimizer steps. Red points mark boundary locations, white dashed curves show power-law fits \(\gamma=C\delta^p\), and solid cyan lines show estimates from the effective loss exponents in Table~\ref{tab:wall-exponents} and the empirical relation \(C_L=12(n-2)/(3n-4)\) in Section~\ref{sec:one-dimensional-toy-model-of-adam-stability}. Gray cells have no accepted period. Unsampled regions are uncolored.}
\label{fig:six-boundaries}
\end{figure}

We scan logarithmically spaced memory gaps \(\delta=1-\beta_1\) and \(\gamma=1-\beta_2\). All six diagrams cover \(\delta\in[10^{-3},0.2]\) and \(\gamma\in[5\times10^{-4},0.3]\). After the initial transient, we extend trajectories with unresolved periods and identify macroscopic oscillations from low-frequency spectral peaks that meet the prominence criterion in Appendix~\ref{a.3-period-detection-algorithm}.

For each fixed \(\gamma\), the boundary is the immediate right neighbor of the largest sampled \(\delta\) with no accepted period. We fit these points over the selected ranges using \(\log\gamma=\log C+p\log\delta\) to estimate \(C\) and \(p\). \textbf{All six phase diagrams approach the linear scaling \(\gamma\propto\delta\)}, with fitted exponents \(p=0.92,0.84,0.75,0.95,1.05,0.95\) for (a--f), respectively. The region above the boundary, where \((1-\beta_2)/(1-\beta_1)\) is larger, generally contains no detected macroscopic spikes. Within the spiky region, periods tend to increase as \(\beta_2\) approaches one and depend more weakly on \(\beta_1\).

\section{One-dimensional toy model of Adam stability}\label{sec:one-dimensional-toy-model-of-adam-stability}
To isolate the boundary mechanism, we first state the relevant local stability condition.
\begin{lemma}[Local stability condition for Adam]\label{lem:adam-local-stability}
Consider a network in a locally quadratic region with Hessian \(H_t\) and preconditioner \(D_t=\operatorname{diag}[(\sqrt{\hat v_t}+\epsilon)^{-1}]\). With both held fixed locally and late-time bias corrections neglected, Adam is unstable if the following condition holds \cite{cohen2023adaptive,bai2026adaptive}:
\begin{equation}
  \frac{1-\beta_1}{1+\beta_1}\lambda_{\max}(D_tH_t)>\frac{2}{\eta}.
\end{equation}
\end{lemma}

Appendix~\ref{proof:adam-local-stability} gives the derivation. Its one-dimensional specialization follows immediately.

\begin{corollary}[Stability condition for Adam in 1D]\label{cor:adam-stability-1d}
For one-dimensional curvature \(\lambda_t=L''(x_t)>0\), define
\begin{equation}
  \eta_{\rm eff,t}=\frac{\eta}{\sqrt{\hat v_t}+\epsilon},
  \qquad
  \eta_{\rm crit,t}=\frac{2(1+\beta_1)}{(1-\beta_1)\lambda_t}.
\end{equation}
Under the same local approximation, Adam is unstable when \(\eta_{\rm eff,t}>\eta_{\rm crit,t}\).
\end{corollary}

We first study the quadratic loss \(L(x)=kx^2/2\) using Adam with \(k=1\), learning rate \(\eta=0.1\), numerical stabilizer \(\epsilon=10^{-30}\), no weight decay, \(x_0=1\), and \(m_0=v_0=0\). Each beta pair is trained for 250,000 updates, of which the first 125,000 are discarded. We estimate the period from upward crossings of the local stability threshold and set \(T=125{,}000/N_{\rm cross}\) when at least three crossings are detected. Figure~\ref{fig:toy-boundaries}(a) shows the approximately cubic boundary \(\gamma\simeq15.6\delta^{3.07}\), consistent with Figure~\ref{fig:overview}(a). Repeated crossings occur mainly below this boundary. Appendix~\ref{effect-of-weight-decay-on-one-dimensional-toy-model} examines the effect of weight decay on 1-dimensional models.

The boundary remains approximately cubic across the tested \(k\) and \(\eta\), and also across \(\epsilon\) while \(\epsilon\ll\sqrt{\hat v_t}\). Appendix~\ref{sensitivity-of-adam-on-lxkx22} reports these controls, which support an explanation based on the relative memory timescales of the two moments. Because the quadratic exponent differs substantially from the near-unit neural-network exponents, we next consider \(L(x)=k|x|^n/n\). For \(1<n<2\), divergent curvature near zero produces rapid oscillations rather than separated macroscopic spikes. We therefore focus on the superquadratic case \(n>2\).

\begin{figure}[t]
\centering
\includegraphics[width=\linewidth,keepaspectratio]{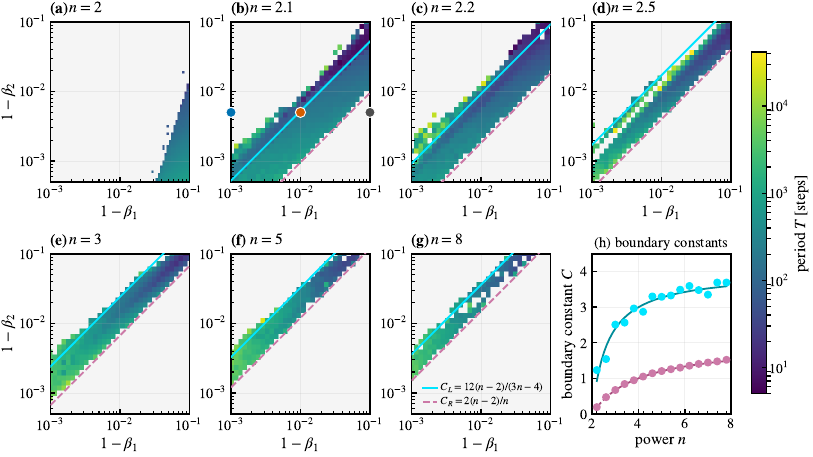}
\caption{Phase diagrams for the one-dimensional losses \(L(x)=k|x|^n/n\) with \(k=1\). (a--g) Results for \(n=2,2.1,2.2,2.5,3,5,8\), respectively, where \(\delta=1-\beta_1\), \(\gamma=1-\beta_2\), and color indicates the period. Solid cyan lines show the empirical left boundary \(\gamma=12(n-2)\delta/(3n-4)\), and dashed pink lines show the predicted right boundary \(\gamma=2(n-2)\delta/n\). Both coefficients vanish at \(n=2\), where the boundary is approximately cubic. The blue, orange, and gray points in (b) identify the trajectories in Figure~\ref{fig:eta-competition}. (h) Measured boundary coefficients as functions of \(n\).}
\label{fig:toy-boundaries}
\end{figure}

\textbf{For \(n>2\), Figure~\ref{fig:toy-boundaries}(b--g) shows a wedge-shaped spiky region between two lines of approximately unit log--log slope,}
\begin{equation}
C_L(1-\beta_1)\geq1-\beta_2\geq C_R(1-\beta_1).
\end{equation}
We derive the right boundary from Adam's local stability condition and characterize the left boundary empirically. The boundaries observed in the six neural-network phase diagrams correspond to the left boundary \(C_L\). The right boundary \(C_R\) requires the loss to remain superquadratic asymptotically as the displacement approaches zero. As Section~\ref{sec:core-wall-landscape-and-its-exponential-sum-mechanism} shows, the measured neural-network landscapes instead enter a quadratic core in this limit, so the right-boundary mechanism does not apply to them.

During a quiet interval, \((1-\beta_2)g_{t+1}^2\ll\beta_2v_t\) makes self-decay dominate the second moment, giving \(v_t\propto\beta_2^t\) \citep{liu2026optimization-4}. Lemma~\ref{lem:second-moment-self-decay} makes this decay asymptotically exact for bounded rescaled trajectories. While \(\sqrt{\hat v_t}\gg\epsilon\), this decay raises \(\eta_{\rm eff,t}\). The approach to zero on a superquadratic loss simultaneously lowers the curvature and raises \(\eta_{\rm crit,t}\). Figure~\ref{fig:eta-competition}(d--f) illustrates three resulting regimes. If \(\eta_{\rm crit,t}\) grows faster, \(\eta_{\rm eff,t}\) cannot catch it and sustained macroscopic spikes are suppressed, as in panel (f). If \(\eta_{\rm eff,t}\) repeatedly reaches the threshold, each crossing triggers a spike whose large gradients replenish \(v_t\) and lower \(\eta_{\rm eff,t}\), as in panel (e). Alternatively, the trajectory can be captured by a short-period microscopic attractor, where \(\eta_{\rm eff,t}\) saturates below \(\eta_{\rm crit,t}\) and rapid oscillations persist without macroscopic spikes, as in panel (d). Theorem~\ref{thm:right-boundary} formalizes the first regime for trajectories attracted to the rescaled fixed points. We absorb \(1/n\) into \(k\) and write \(L(x)=k|x|^n\), which leaves the boundary coefficients unchanged.

\begin{theorem}[Right boundary for \(L(x)=k|x|^n\) models]\label{thm:right-boundary}
Consider Adam applied to the loss landscape \(L(x)=k|x|^n\), where \(\frac12<\beta_1<1\), \(0<\beta_2<1\), and \(n>2\). In the late-training regime, neglect \(\epsilon\) and the bias-correction factors, then trajectories eventually cease to exhibit EoS-triggered spikes when
\begin{equation}
  \beta_1\beta_2^{-n/[2(n-2)]}<1.
\end{equation}
\end{theorem}

\begin{proof}[Proof sketch]
An equivalent condition was derived by \citet{bai2026degenerate} from the local stability of the full normalized Adam dynamics on even-degree polynomials. Lemma~\ref{lem:second-moment-self-decay} shows that the full update is a summable perturbation of the limiting recurrence below for every real \(n>2\).
\begin{equation}
y_{t+1}=a y_t-b y_{t-1}-c|y_t|^{n-2}y_t,
\qquad
b=\beta_1\beta_2^{-n/[2(n-2)]},
\end{equation}
where \(a\), \(b\), and \(c\) are time independent. Whenever the nonzero fixed points exist within the theorem's parameter range, both satisfy \(S_*<1\). Linearization and the Jury conditions show that these fixed points are stable exactly when \(b<1\), so attracted trajectories eventually cease to spike. The boundary \(b=1\) gives \(\beta_2=\beta_1^{2(n-2)/n}\). Appendix~\ref{proof:right-boundary} provides the fixed-point bounds and full stability calculation.
\end{proof}

Along these trajectories, \(x_t\sim q^t y_*\) gives \(g_t^2/v_t\propto q^{2t}\to0\), consistent with the self-decay approximation. In the phase diagram coordinates, the right boundary is
\begin{equation}
\gamma_R(\delta,n)=1-(1-\delta)^{2(n-2)/n}=\frac{2(n-2)}{n}\delta+O(\delta^2).
\end{equation}
The plotted line uses the leading small-\(\delta\) coefficient \(C_R=2(n-2)/n\). The \(n=2\) boundary is not linear and instead exhibits the approximately cubic scaling in Figure~\ref{fig:toy-boundaries}(a).

\begin{figure}[t]
\centering
\includegraphics[width=\linewidth,keepaspectratio]{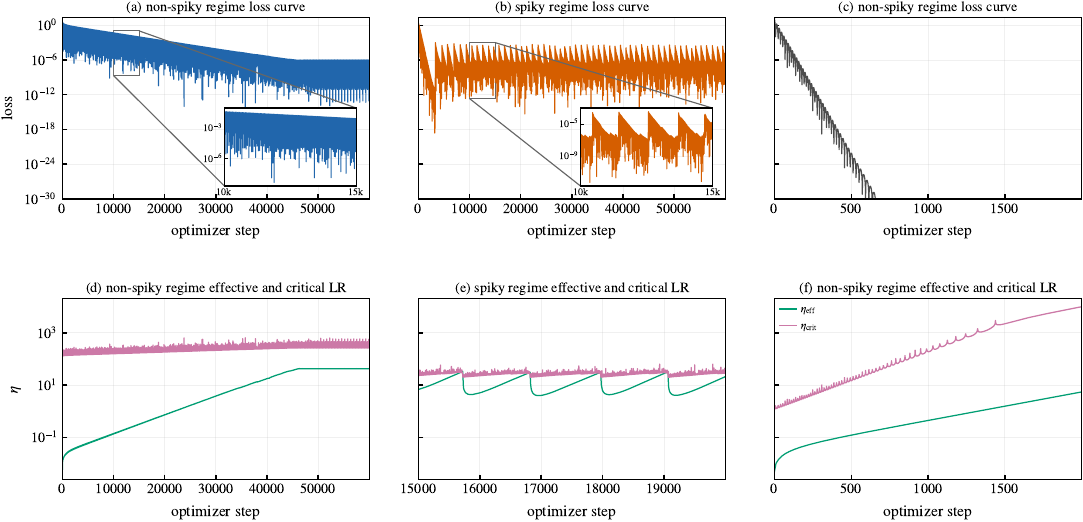}
\caption{Loss and learning-rate dynamics for \(n=2.1\). The columns correspond to \((\beta_1,\beta_2)=(0.999,0.995)\), \((0.99,0.995)\), and \((0.9,0.995)\), matching the blue, orange, and gray points in Figure~\ref{fig:toy-boundaries}(b). (a--c) Loss curves. (d--f) Effective learning rate \(\eta_{\rm eff}\) in green and critical rate \(\eta_{\rm crit}\) in pink. In (d), rapid oscillations persist while \(\eta_{\rm eff}\) remains below \(\eta_{\rm crit}\) and eventually saturates. In (e), \(\eta_{\rm eff}\) repeatedly reaches the threshold, producing spikes that reduce it. In (f), faster growth of \(\eta_{\rm crit}\) suppresses sustained macroscopic spikes.}
\label{fig:eta-competition}
\end{figure}
These regimes show that the left boundary involves more complex oscillatory dynamics and finite-step effects, and a complete analytical explanation remains open. We summarize its dependence on \(n\) by the empirical relation \(C_L\simeq12(n-2)/(3n-4)\), which captures the trend in Figure~\ref{fig:toy-boundaries}(h).

\section{Core--wall landscape and its exponential-sum mechanism}\label{sec:core-wall-landscape-and-its-exponential-sum-mechanism}

The one-dimensional results suggest that the near-linear \(\beta\)-boundary reflects a superquadratic loss profile at the scale explored by Adam. Although the conventional quadratic approximation remains valid sufficiently close to a low-loss minimum, its domain of validity can be much smaller than an optimizer update \cite{ma2022beyond}. The optimizer then probes the steep finite-scale growth outside this neighborhood, motivating our \textbf{core--wall landscape} description. It refines the usual river--valley picture by resolving the basin around a low-loss minimum into an inner quadratic core, a superquadratic wall, and an outer rollover. Related notions of flatness have been connected to generalization \cite{hochreiter1997flat,keskar2017largebatch}, while basin-like parameter regions have been used to study behavioral preservation and fine-tuning robustness \cite{peng2024safety,chen2026basin}. \textbf{We quantify this directional geometry using an effective logarithmic exponent}. 
Experimentally, we observe the core--wall landscape at pre-spike points in all six model--task settings.  Appendix~\ref{a.4-locating-the-pre-spike-points} details the pre-spike selection procedure and illustrates how contraction of the quadratic core toward the optimizer's update scale produces progressive sharpening at finite scales.

\begin{definition}[Core--wall landscape]\label{def:core-wall-landscape}
For a frozen direction \(\hat d\), define the directional slice \(\phi(s)=L(\theta+s\hat d)\) and let \(s_*\) be one of its local minima. For \(r>0\), define the two excess-loss branches and their local exponents by
\begin{equation}
\Delta L_\pm(r)=\phi(s_*\pm r)-\phi(s_*),
\qquad
n_{{\rm eff},\pm}(r)=\frac{\mathrm d\log\Delta L_\pm(r)}{\mathrm d\log r}.
\end{equation}
As \(r\) increases, the branches pass successively through a quadratic core with \(n_{{\rm eff},\pm}\simeq2\), a superquadratic wall beginning at the first crossing of the operational threshold \(n_{{\rm eff},\pm}=2.5\), and an outer rollover in which \(n_{{\rm eff},\pm}\) decreases from its wall value. We call this ordered basin structure the \emph{core--wall landscape}.
\end{definition}

To quantify this observation across the six tasks in Figure~\ref{fig:six-boundaries}, we evaluate low-loss checkpoints along the Adam-preconditioned gradient direction
\begin{equation}
\phi_t(s)=L(\theta_t+s\hat d_t),
\qquad
\hat d_t=\frac{D_tg_t}{\lVert D_t^{1/2}g_t\rVert_2}.
\end{equation}
During descent and near the bottoms of temporal loss valleys, these slices exhibit a core--wall landscape. Conventional progressive sharpening tracks the growth of Hessian or preconditioned-Hessian sharpness \cite{cohen2021gradient,cohen2023adaptive}. Our slices reveal a finite-scale counterpart in which the quadratic core also contracts toward the optimizer's update scale, as illustrated in Appendix Figure~\ref{fig:three-stages}. We next show that this geometry arises naturally from confident cross-entropy losses. \textbf{The flat core follows from Hessian collapse as the predicted probabilities approach their hard targets.}

\begin{lemma}[Hessian collapse for confident cross-entropy]\label{lem:hessian-collapse}
On a fixed dataset of \(N\) samples with one-hot targets \(y_i\), let \(L(\theta)=N^{-1}\sum_i\ell_i(\theta)\) be the unregularized softmax cross-entropy, with twice differentiable logits \(z_i(\theta)\) and probabilities \(p_i=\operatorname{softmax}(z_i)\). Along any parameter sequence such that \(p_i\to y_i\) for every sample, assume that \(J_i=\partial z_i/\partial\theta\) and \(\nabla_\theta^2z_{ik}\) remain uniformly bounded. Then \(\|\nabla_\theta^2L\|_2\to0\).
\end{lemma}

\begin{proof}[Proof sketch]
The generalized Gauss--Newton decomposition writes each sample Hessian as \(J_i^\top C_iJ_i+\sum_k(p_{ik}-y_{ik})\nabla_\theta^2z_{ik}\), where \(C_i=\operatorname{diag}(p_i)-p_ip_i^\top\). If \(\varepsilon_i=1-p_{i,c_i}\), then \(\|C_i\|_2\leq2\varepsilon_i\) and \(\|p_i-y_i\|_1=2\varepsilon_i\). The assumed derivative bounds therefore make both terms \(O(\varepsilon_i)\), and the finite average vanishes as every \(\varepsilon_i\to0\). Appendix~\ref{proof:hessian-collapse} gives the full norm bounds.
\end{proof}


Lemma~\ref{lem:hessian-collapse} identifies the origin of the flat core but does not determine the size of the neighborhood in which the quadratic approximation is accurate. We derive this scale from the exponential-sum structure of confident cross-entropy.

\begin{definition}[Confident cross-entropy loss slice]\label{def:confident-ce-slice}
Let \(i\in\{1,\ldots,N\}\) index samples or tokens, let \(y_i\) be the correct class for sample \(i\), and let \(j\neq y_i\) index its incorrect classes. Along the frozen slice \(\theta(s)=\theta_*+s\hat d\), define the correct-versus-incorrect logit margin as \(m_{ij}(s)=z_{i,y_i}(s)-z_{ij}(s)\).
The exact cross-entropy slice is
\begin{equation}
L_{\rm CE}(s)=\frac1N\sum_{i=1}^N\log\left(1+\sum_{j\neq y_i}e^{-m_{ij}(s)}\right).
\end{equation}
In the confident regime, the relevant margins are large and positive, so \(\log(1+u)\simeq u\). If they are locally linearized as \(m_{ij}(s)\simeq m_{ij}-a_{ij}s\), where \(m_{ij}=m_{ij}(0)\) and \(a_{ij}=-\mathrm d m_{ij}(s)/\mathrm ds|_{s=0}\), we call
\begin{equation}
L_{\rm conf}(s)=\frac1N\sum_i\sum_{j\neq y_i}e^{-m_{ij}+a_{ij}s}.
\end{equation}
the \emph{confident cross-entropy loss slice}.
\end{definition}

In confident cross-entropy, the large margins \(m\gg1\) suppress the baseline loss, \(L_{\rm conf}(0)\propto\sum e^{-m}\simeq0\), whereas the directional slopes \(a\) can have much larger magnitudes. At the local loss minimum in Figure~\ref{fig:core-wall-rollover}(a), for example, \(\langle m\rangle_w\pm\operatorname{std}_w(m)=17.08\pm1.29\), whereas \(\langle a\rangle_w\simeq1.33\times10^2\) and \(\sigma_a\simeq2.30\times10^5\). Such large slope magnitudes make factors of the form \(e^{-m+as}\) grow exponentially under very small displacements, \textbf{leaving only an \(O(1/\sigma_a)\) neighborhood in which the quadratic approximation remains accurate.} The following theorem makes this scale precise under a Gaussian model for \(a\).

\begin{theorem}[Quadratic-core scale of confident cross-entropy]\label{thm:quadratic-core-scale}
Approximate the CE-weighted distribution of directional margin slopes by a Gaussian, \(a\overset{w}{\sim}\mathcal N(\mu_a,\sigma_a^2)\). Within this Gaussian model, with the origin chosen at the minimum of the confident surrogate and \(\sigma_a>0\), the slice has a quadratic-core scale \(s_{\rm core}=O(1/\sigma_a)\).
\end{theorem}

\begin{proof}[Proof sketch]
The normalized confident loss is the moment-generating function of the weighted slope distribution. Gaussianity and stationarity at \(s=0\) give
\begin{equation}
  \frac{\Delta L_{\rm conf}(s)}{L_{\rm conf}(0)}=e^{(\sigma_as)^2/2}-1,
  \qquad
  n_{\rm eff}=\frac{(\sigma_as)^2e^{(\sigma_as)^2/2}}{e^{(\sigma_as)^2/2}-1}.
\end{equation}
Thus \(n_{\rm eff}=2+(\sigma_as)^2/2+O((\sigma_as)^4)\), so departure from the quadratic core occurs at \(|s|=O(1/\sigma_a)\); the threshold \(n_{\rm eff}=2.5\) gives \(|s|\simeq0.964/\sigma_a\). Appendix~\ref{proof:quadratic-core-scale} contains the full derivation.
\end{proof}

To resolve the local shape, we refine \(s_*\) for each frozen slice and plot \(\Delta L=\phi_t(s)-\phi_t(s_*)\) against \(|s-s_*|\) on logarithmic axes. Figure~\ref{fig:core-wall-rollover} compares a measured modular-division slice with a synthetic confident cross-entropy example. We construct the latter as \(L_{\rm CE}(s)=\log(1+\sum_{i=1}^{16}e^{-M_i+a_i s})\), with \(M_i\sim\operatorname{Uniform}[10,20]\) and \(a_i\sim\mathcal N(0,3000^2)\), then shift the slopes so that \(\sum_i e^{-M_i}a_i=0\) and \(s_*=0\). In both cases, the logarithmic slope rises from an approximately quadratic core to a superquadratic wall and then decreases in the outer rollover.

\begin{figure}[t]
\centering
\includegraphics[width=\linewidth,keepaspectratio]{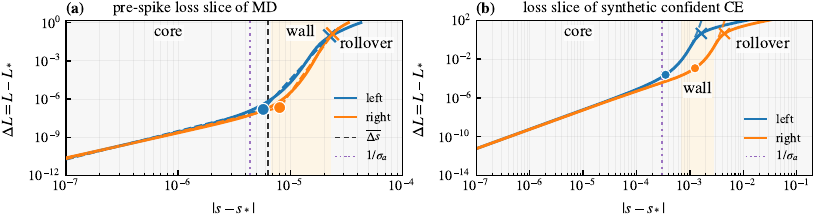}
\caption{Directional loss slices shown as \(\Delta L=L-L_*\) versus \(|s-s_*|\) on log--log axes. (a) The selected modular-division pre-spike point at step \(37{,}210\) for \((\beta_1,\beta_2)=(0.9,0.999)\). Solid curves are the measured left and right branches, dashed curves are piecewise core--wall landscape fits, circles and crosses mark the fitted knots, and the vertical dashed line marks the typical projected update \(\overline{\Delta s}\). Shading identifies the core, wall, and rollover fit intervals. (b) A synthetic confident cross-entropy slice. Solid curves show the exact loss, and dashed curves show its confident approximation. Purple dotted lines mark \(1/\sigma_a\). In both examples, the wall emerges at the exponential-sum scale.}
\label{fig:core-wall-rollover}
\end{figure}

For the measured slice in Figure~\ref{fig:core-wall-rollover}(a), averaging the two branches gives \(n_{\rm eff}\simeq2.18\) in the core and \(n_{\rm eff}\simeq10.90\) along the wall. The core--wall transition occurs at \(O(10^{-6})\), comparable to the typical update scale \(\overline{\Delta s}\simeq6.3\times10^{-6}\), while the rollover begins at \(O(10^{-5})\). Consistently, \(1/\sigma_a\simeq4.35\times10^{-6}\) for the measured slice. The synthetic slice in Figure~\ref{fig:core-wall-rollover}(b) exhibits the same three regimes with \(1/\sigma_a\simeq2.97\times10^{-4}\).

Beyond the core, exponential reweighting favors terms with increasingly extreme directional slopes. If \(\langle a\rangle_s\) denotes their loss-weighted mean at displacement \(s\), then
\begin{equation}
\frac{\mathrm d\log L_{\rm conf}}{\mathrm d\log|s|}=s\langle a\rangle_s.
\end{equation}
The growth of \(|\langle a\rangle_s|\) produces a superquadratic wall that admits a finite-interval power-law approximation, yielding the effective exponent used in the one-dimensional model. At larger displacements, the logarithm in exact cross-entropy changes a dominant contribution from exponential to asymptotically linear growth, reducing the effective exponent and producing the rollover.

Finally, we connect these directional measurements to the \(\beta\)-phase diagrams. We use the fitted wall exponent over the resolved intermediate regime as an effective input to the empirical left-boundary relation from Section~\ref{sec:one-dimensional-toy-model-of-adam-stability},
\begin{equation}
C_L(n_{\rm wall})\simeq\frac{12(n_{\rm wall}-2)}{3n_{\rm wall}-4},
\qquad 1-\beta_2=C_L(n_{\rm wall})(1-\beta_1).
\end{equation}
The non-spiky region below the right boundary of the pure power-law model arises asymptotically as \(s\to0\) while the loss remains superquadratic. In the measured neural-network landscapes, however, this limit enters the quadratic core, cutting off the pure-superquadratic right-boundary mechanism before it can be observed. We therefore compare the measured wall exponent only with the empirical left boundary.

For a controlled comparison, all six model--task settings use the same six reference beta pairs and apply the same core--wall--rollover fitting procedure at the pre-spike points. We first obtain one jointly fitted \(n_{\rm wall}\) for each beta pair and then report the unweighted mean over the six pairs in Table~\ref{tab:wall-exponents}. We evaluate \(C_L\) from this mean exponent. Appendix~\ref{a.4-locating-the-pre-spike-points} describes how the pre-spike points are located, and Appendix~\ref{a.5-fitting-and-aggregating-the-wall-exponent} gives the directional-slice fit and aggregation procedure.

\begin{table}[t]
\caption{Effective wall exponents and the resulting empirical left-boundary coefficients for the six model--task settings.}
\label{tab:wall-exponents}
\centering
\begin{tabular}{@{}lcc@{}}
\toprule
Model--task setting & \(n_{\rm wall}\) & Estimated \(C_L\) \\
\midrule
Modular division & 10.90 & 3.72 \\
Modular addition & 10.58 & 3.71 \\
Tiny Shakespeare & 9.41 & 3.67 \\
CIFAR-10 MLP & 10.48 & 3.71 \\
CIFAR-10 CNN & 6.52 & 3.49 \\
MNIST autoencoder & 10.22 & 3.70 \\
\bottomrule
\end{tabular}
\end{table}

\section{Conclusion and outlook}\label{sec:conclusion}

Taken together, our results make three contributions. First, they establish a near-linear phase boundary between spiky and non-spiky Adam dynamics across diverse models and tasks. Second, our analysis attributes the shift from cubic to linear boundary scaling to superquadratic loss geometry and shows how the first- and second-moment memory timescales set the phase boundaries. Third, we study local loss landscapes at finite scales and show that confident cross-entropy produces a core--wall structure consisting of a narrow quadratic core, a superquadratic wall, and an outer rollover. The effective exponent measured at the optimizer's update scale predicts the empirical boundary coefficient, directly linking finite-scale landscape geometry to the beta phase diagram.

Looking ahead, three questions appear especially promising.
\begin{enumerate}
\item Basin- and valley-like loss landscapes arise in broader underparameterized neural-network settings \cite{bosman2020visualising,ruizgarcia2021tilting}. Can the core--wall exponential-sum mechanism also explain these landscapes?
\item More generally, non-quadratic geometry with an exponent \(n_{\rm eff}\) that changes across direction, scale, and training time is likely to be the norm. Can an optimizer estimate \(n_{\rm eff}\) at its update scale and adapt its learning rate or moment timescales to the evolving phase boundary?
\item What mechanism drives progressive sharpening at finite scales? How do the margins \(m_{ij}\) and their directional derivatives \(a_{ij}=-\mathrm{d}m_{ij}/\mathrm{d}s\) evolve during training?
\end{enumerate}

\section{Related work}\label{sec:related-literature-and-relationship-to-existing-work}

\paragraph{Edge of Stability (EoS).}
Trajectory studies linked learning rates to sharp directions and showed that progressive sharpening reaches the quadratic threshold \(2/\eta\), producing EoS dynamics \cite{xing2018walk,jastrzebski2018relation,jastrzebski2020breakeven,cohen2021gradient}. Adaptive EoS uses the preconditioned Hessian, central-flow theory describes averaged oscillations, and decoupling between \(g_t^2\) and \(v_t\) can trigger Adam spikes \cite{cohen2023adaptive,cohen2025centralflows,bai2026adaptive}. Low-dimensional analyses cover quadratic two-cycles and frozen-EoS restoration \cite{bock2019nonconvergence,fong2026provable}. \citet{bai2026degenerate} analyze the \((\beta_1,\beta_2)\) phase transition on even-degree degenerate polynomials, whereas we study it on the core--wall landscapes observed in real models. Owing to the quadratic core, at fixed \(\beta_1\) the spiky phase in real models occurs at larger \(\beta_2\), in direct contrast to their prediction. Complementary work studies near-zero Hessian bulk, loss slices, and river-valley dynamics \cite{sagun2018hessian,li2018visualizing,wen2024river}.

\paragraph{Loss spikes.}
Loss spikes in large-language-model training may require checkpoint-and-data interventions \cite{chowdhery2023palm}, while unusually large Adam updates can become weakly aligned with descent \cite{molybog2023adam}. Momentum-dependent oscillations, spikes, and divergence have also been observed \cite{ma2022qualitative}. Proposed mechanisms include lower-loss-as-sharper geometry and lagging second moments that amplify preconditioned curvature even on quadratics \cite{li2023lossspike,bai2026adaptive}. Slingshot connects cyclic adaptive-optimizer instability to grokking, although neither implies the other \cite{power2022grokking,thilak2022slingshot}. Numerical Feature Inflation and weight-norm criticality explain spikes in low precision or scale-invariant networks with weight decay \cite{liu2026grokking,li2026weightnorm}. Our one-dimensional examples attribute beta-boundary scaling to competing moment timescales and superquadratic finite-scale geometry.

\paragraph{Core--wall landscapes.}
Flat minima have been associated with generalization \cite{hochreiter1997flat,keskar2017largebatch}, although parameter-space sharpness is not invariant to function-preserving reparameterizations \cite{dinh2017sharp}. Volume flatness, worst-case loss increases, and Hessian sharpness address different questions. Recent LLM studies define basins by preserving alignment or task performance under parameter perturbations \cite{peng2024safety,chen2026basin}. These characterize fine-tuning robustness, whereas our core--wall landscape uses local training loss along a dynamically selected direction to explain spike onset.

\paragraph{Optimization with non-quadratic curvature.}
Beyond the quadratic threshold, prior work studies unstable convergence, multiscale subquadratic landscapes, and cubic self-stabilization \cite{ahn2022unstable,ma2022beyond,damian2023selfstabilization}. Generalized self-concordance bounds finite-displacement departures for logistic loss, while the generalized Gauss--Newton decomposition shows why pointwise curvature need not determine finite-step dynamics \cite{bach2010selfconcordant,schraudolph2002fast}. Neither derives the directional scale \(O(1/\sigma_a)\) of a confident cross-entropy exponential sum. 

\FloatBarrier
\clearpage
\bibliography{references}
\bibliographystyle{plainnat}

\appendix

\section{Proofs}\label{proofs}

\subsection{Proof of Lemma~\ref{lem:adam-local-stability}}\label{proof:adam-local-stability}

\begin{proof}
For gradient descent, the local perturbation vector evolves as
\begin{equation}
  \delta_{t+1}=(I-\eta H_t)\delta_t.
\end{equation}
Along an eigenvector of \(H_t\) with eigenvalue \(\lambda\), this becomes
\begin{equation}
  \delta_{t+1}=(1-\eta\lambda)\delta_t.
\end{equation}
Stability requires \(|1-\eta\lambda|<1\), which gives
\begin{equation}
  \eta\lambda<2
  \qquad\Longrightarrow\qquad
  \lambda<\frac{2}{\eta}.
\end{equation}

With first-moment momentum,
\begin{equation}
  m_t=\beta_1m_{t-1}+(1-\beta_1)g_t,\qquad \theta_{t+1}=\theta_t-\eta m_t.
\end{equation}
Linearizing \(g_t\approx H_t\delta_t\) and eliminating \(m_t\) gives the vector recurrence
\begin{equation}
  \delta_{t+1}=\left[(1+\beta_1)I-\eta(1-\beta_1)H_t\right]\delta_t-\beta_1\delta_{t-1}.
\end{equation}
Along a Hessian eigen-direction, the characteristic equation is
\begin{equation}
  r^2-\alpha r+\beta_1=0,\qquad \alpha=(1+\beta_1)-\eta(1-\beta_1)\lambda.
\end{equation}
The Jury stability condition gives
\begin{equation}
  \eta(1-\beta_1)\lambda<2(1+\beta_1),
\end{equation}
or equivalently
\begin{equation}
  \frac{1-\beta_1}{1+\beta_1}\lambda<\frac{2}{\eta}.
\end{equation}

Adam replaces the raw gradient by a coordinate-wise normalized gradient:
\begin{equation}
  v_t=\beta_2v_{t-1}+(1-\beta_2)g_t^2,\qquad \theta_{t+1}=\theta_t-\eta\frac{\hat m_t}{\sqrt{\hat v_t}+\epsilon}.
\end{equation}
Because \(v_t\) tracks squared gradients, its RMS scale is \(\sqrt{\hat v_t}\), giving the preconditioner
\begin{equation}
  D_t=\operatorname{diag}\left(\frac{1}{\sqrt{\hat v_t}+\epsilon}\right).
\end{equation}
Freezing \(D_t\) and \(H_t\) replaces the local operator \(H_t\) by \(D_tH_t\). In the absence of first-moment momentum, the corresponding perturbation dynamics are
\begin{equation}
  \delta_{t+1}\approx(I-\eta D_tH_t)\delta_t.
\end{equation}
The matrix \(D_tH_t\) is similar to the symmetric matrix \(D_t^{1/2}H_tD_t^{1/2}\), so its eigenvalues are real. Including first-moment momentum and applying the preceding condition to the largest eigenvalue yields the instability criterion
\begin{equation}
  \frac{1-\beta_1}{1+\beta_1}\lambda_{\max}(D_tH_t)>\frac{2}{\eta},
\end{equation}
which is the Adam EoS instability condition.
\end{proof}

\subsection{Proof of Theorem~\ref{thm:right-boundary}}\label{proof:right-boundary}

\begin{lemma}[Second-moment self-decay]\label{lem:second-moment-self-decay}
Let \(g_t=kn|x_t|^{n-2}x_t\), \(v_t=\beta_2v_{t-1}+(1-\beta_2)g_t^2\), and \(v_0>0\). Set \(q=\beta_2^{1/[2(n-2)]}\), \(x_t=q^t y_t\), and \(V_t=\beta_2^{-t}v_t\). If \(\sup_t|y_t|<\infty\), then \(V_t\) converges to a finite limit \(V_\infty>0\), and
\begin{equation}
v_t=V_\infty\beta_2^t[1+O(q^{2t})],
\qquad
\frac{g_t^2}{v_t}=O(q^{2t}),
\qquad
\frac{v_t}{\beta_2v_{t-1}}=1+O(q^{2t}).
\end{equation}
\end{lemma}

\begin{proof}
The rescaled second moment obeys the exact recurrence
\begin{equation}
V_t-V_{t-1}
=(1-\beta_2)\beta_2^{-t}g_t^2
=(1-\beta_2)(kn)^2q^{2t}|y_t|^{2(n-1)}.
\end{equation}
If \(|y_t|\leq M\), the increments are nonnegative and bounded by a summable geometric sequence. Hence
\begin{equation}
V_\infty
=v_0+(1-\beta_2)(kn)^2
\sum_{s=1}^{\infty}q^{2s}|y_s|^{2(n-1)}
\end{equation}
exists, is finite, and is strictly positive. Moreover,
\begin{equation}
0\leq V_\infty-V_t
\leq
\frac{(1-\beta_2)(kn)^2M^{2(n-1)}}{1-q^2}q^{2(t+1)}.
\end{equation}
Thus \(V_t=V_\infty[1+O(q^{2t})]\), which proves the first claim after multiplying by \(\beta_2^t\). The remaining two claims follow from
\begin{equation}
\frac{g_t^2}{v_t}
=\frac{(kn)^2q^{2t}|y_t|^{2(n-1)}}{V_t},
\qquad
\frac{v_t}{\beta_2v_{t-1}}=\frac{V_t}{V_{t-1}},
\end{equation}
because \(V_t\geq v_0>0\) and \(V_t-V_{t-1}=O(q^{2t})\).
\end{proof}

\begin{proof}[Proof of Theorem~\ref{thm:right-boundary}]
Using the notation of Lemma~\ref{lem:second-moment-self-decay}, a trajectory covered by the theorem satisfies \(\eta_{\rm eff,t}=\eta/\sqrt{v_t}\), \(V_t\to V_\infty>0\), and \(V_\infty-V_t=O(q^{2t})\). Define the stability ratio
\begin{equation}
S_t=\frac{\eta_{\rm eff,t}}{\eta_{\rm crit,t}}=
\frac{(1-\beta_1)kn(n-1)\eta}{2(1+\beta_1)\sqrt{V_t}}|y_t|^{n-2}.
\end{equation}
A trajectory that eventually remains at \(S_t<1\) has no further threshold crossings and hence no further spikes triggered by this mechanism.

Eliminating the first moment from the exact discrete updates gives
\begin{equation}
\begin{aligned}
x_{t+1}={}&\left(1+\beta_1\sqrt{\frac{v_{t-1}}{v_t}}\right)x_t
-\beta_1\sqrt{\frac{v_{t-1}}{v_t}}x_{t-1}\\
&-(1-\beta_1)kn\frac{\eta}{\sqrt{v_t}}|x_t|^{n-2}x_t.
\end{aligned}
\end{equation}
Under this rescaling, the dynamics become the asymptotically autonomous recurrence
\begin{equation}
y_{t+1}=a_t y_t-b_t y_{t-1}-c_t|y_t|^{n-2}y_t,
\end{equation}
where
\begin{equation}
R_t=\sqrt{\frac{V_{t-1}}{V_t}},
\qquad
a_t=\frac{1+\beta_1\beta_2^{-1/2}R_t}{q},
\qquad
b_t=\frac{\beta_1R_t}{q^n},
\qquad
c_t=\frac{(1-\beta_1)kn\eta}{q\sqrt{V_t}}.
\end{equation}
The lemma gives \(R_t=1+O(q^{2t})\), so these coefficients converge at a summable rate to
\begin{equation}
a=\frac{1+\beta_1\beta_2^{-1/2}}q,
\qquad
b=\frac{\beta_1}{q^n}=\beta_1\beta_2^{-n/[2(n-2)]},
\qquad
c=\frac{(1-\beta_1)kn\eta}{q\sqrt{V_\infty}}.
\end{equation}

The two nonzero fixed points exist when \(a-1-b>0\), equivalently \(bq<1\), and satisfy
\begin{equation}
y_*=\pm\left(\frac{a-1-b}{c}\right)^{1/(n-2)}.
\end{equation}
Both have the same stability ratio because \(S\) depends only on \(|y|\):
\begin{equation}
S_*=\frac{(n-1)q(a-1-b)}{2(1+\beta_1)}=
\frac{(n-1)(1-q)(1-bq)}{2(1+\beta_1)}.
\end{equation}
Existence implies \(q>\beta_1^{1/(n-1)}\), while \(q<1\) implies \(bq>\beta_1\). Therefore,
\begin{equation}
S_*<
\frac{(n-1)(1-\beta_1^{1/(n-1)})(1-\beta_1)}{2(1+\beta_1)}<
\frac{(-\ln\beta_1)(1-\beta_1)}{2(1+\beta_1)}<
\frac{\ln2}{6}<1.
\end{equation}

To establish stability, write \(y_t=y_*+u_t\). Expansion around either fixed point yields
\begin{equation}
u_{t+1}=\bigl[(n-1)(1+b)-(n-2)a\bigr]u_t-bu_{t-1}+O(u_t^2),
\end{equation}
with characteristic equation
\begin{equation}
\lambda^2-\bigl[(n-1)(1+b)-(n-2)a\bigr]\lambda+b=0.
\end{equation}
The roots satisfy \(|\lambda_\pm|<1\) if and only if
\begin{equation}
b<1,
\qquad 0<(n-2)(a-1-b)<2(1+b).
\end{equation}
Within the stated parameter range, the second condition follows automatically from \(b<1\). Indeed,
\begin{equation}
q>\frac1{\sqrt2},
\qquad 1-q<\frac{\ln2}{n},
\qquad 0<1-bq<\frac12,
\end{equation}
so
\begin{equation}
0<(n-2)(a-1-b)=
(n-2)\frac{(1-q)(1-bq)}q<
\frac{n-2}{n}\frac{\ln2}{\sqrt2}<2(1+b).
\end{equation}
Conversely, \(|\lambda_\pm|<1\) requires \(b=\lambda_+\lambda_-<1\). At \(b=1\), the fixed points still exist and the same bound gives \(0<(n-2)(a-1-b)<4\), placing the conjugate roots on the unit circle. Immediately beyond this boundary, at least one root has modulus greater than one.

It remains to transfer this stability from the limiting recurrence to the exact time-dependent one. Let \(z_t=(y_t,y_{t-1})\), let \(z_*=(y_*,y_*)\), and let \(A\) be the Jacobian of the limiting two-dimensional map at \(z_*\). When \(b<1\), its spectral radius is less than one, so there is an equivalent norm and a number \(\rho<1\) for which the limiting nonlinear map is locally Lipschitz with constant at most \(\rho\) around \(z_*\). Because \(a_t-a\), \(b_t-b\), and \(c_t-c\) are all \(O(q^{2t})\), the exact recurrence satisfies, after shrinking the neighborhood if necessary,
\begin{equation}
\|z_{t+1}-z_*\|\leq
\rho\|z_t-z_*\|+Cq^{2t}.
\end{equation}
Iteration gives
\begin{equation}
\|z_t-z_*\|\leq
\rho^{t-T}\|z_T-z_*\|
+C\sum_{s=T}^{t-1}\rho^{t-1-s}q^{2s}
\longrightarrow0
\end{equation}
for every trajectory entering this local basin at a sufficiently large time \(T\). Thus the summable perturbation from gradient replenishment does not change the local stability boundary.

Thus, for \(b<1\), trajectories attracted to either fixed point satisfy \(S_t\to S_*<1\) and eventually cease crossing the EoS threshold. The loss of strict linear stability occurs at \(b=1\), giving \(\beta_2=\beta_1^{2(n-2)/n}\).
\end{proof}

\subsection{Proof of Lemma~\ref{lem:hessian-collapse}}\label{proof:hessian-collapse}

\begin{proof}
The generalized Gauss--Newton decomposition \cite{schraudolph2002fast} gives
\begin{equation}
\nabla_\theta^2\ell_i=
\underbrace{J_i^\top C_iJ_i}_{G_i}+
\underbrace{\sum_k(p_{ik}-y_{ik})\nabla_\theta^2z_{ik}}_{R_i},
\qquad C_i=\operatorname{diag}(p_i)-p_ip_i^\top.
\end{equation}
Let \(c_i\) be the correct class and \(\varepsilon_i=1-p_{i,c_i}\). Since \(C_i\succeq0\),
\begin{equation}
\|C_i\|_2\leq\operatorname{tr}C_i=
1-\|p_i\|_2^2\leq1-(1-\varepsilon_i)^2\leq2\varepsilon_i,
\qquad \|p_i-y_i\|_1=2\varepsilon_i.
\end{equation}
Writing \(B_{1,i}=\|J_i\|_2\) and \(B_{2,i}=\max_k\|\nabla_\theta^2z_{ik}\|_2\), submultiplicativity and the triangle inequality yield
\begin{equation}
\|G_i\|_2\leq2\varepsilon_iB_{1,i}^2,
\qquad \|R_i\|_2\leq2\varepsilon_iB_{2,i},
\qquad
\|\nabla_\theta^2L\|_2\leq
\frac{2}{N}\sum_i\varepsilon_i(B_{1,i}^2+B_{2,i})\longrightarrow0.
\end{equation}

The same argument applies to coordinate-wise binary cross-entropy. For a binary target \(y_{ik}\in\{0,1\}\), define the signed margin \(m_{ik}(s)=(2y_{ik}-1)z_{ik}(s)\). Then
\begin{equation}
L_{\rm BCE}(s)=\frac{1}{ND}\sum_{i,k}\log\left(1+e^{-m_{ik}(s)}\right),
\end{equation}
which is a two-class cross-entropy for each sample--coordinate pair. Hence, if every sigmoid prediction approaches its binary target and the corresponding logit derivatives remain bounded, the preceding Hessian-collapse bound holds after replacing \(i\) by the compound index \((i,k)\). Moreover, in the confident regime, locally linearizing \(m_{ik}(s)\simeq m_{ik}-a_{ik}s\) gives \(L_{\rm BCE}(s)\simeq (ND)^{-1}\sum_{i,k}e^{-m_{ik}+a_{ik}s}\), exactly the exponential-sum form in Definition~\ref{def:confident-ce-slice}. If reconstruction remains imperfect so that some pixel probabilities do not approach their targets, however, the vanishing-Hessian conclusion need not hold and a nonzero-curvature core may remain.
\end{proof}

\subsection{Proof of Theorem~\ref{thm:quadratic-core-scale}}\label{proof:quadratic-core-scale}

\begin{proof}
The ratio \(L_{\rm conf}(s)/L_{\rm conf}(0)\) is the moment-generating function of the CE-weighted slope distribution. Replacing this distribution by the Gaussian model gives
\begin{equation}
\frac{L_{\rm conf}(s)}{L_{\rm conf}(0)}=
\mathbb E_w[e^{as}]=
\exp\left(\mu_as+\frac{\sigma_a^2s^2}{2}\right).
\end{equation}
Because \(s=0\) is the local minimum of the slice,
\begin{equation}
0=\frac{L_{\rm conf}'(0)}{L_{\rm conf}(0)}=
\mathbb E_w[a]=\mu_a.
\end{equation}
Therefore, with \(u=\sigma_as\),
\begin{equation}
\frac{\Delta L_{\rm conf}(s)}{L_{\rm conf}(0)}=e^{u^2/2}-1.
\end{equation}
Define the effective log--log exponent by
\begin{equation}
n_{\rm eff}(s)=
\frac{\mathrm d\log\Delta L_{\rm conf}(s)}{\mathrm d\log|s|}.
\end{equation}
The Gaussian model gives
\begin{equation}
n_{\rm eff}(u)=
\frac{u^2e^{u^2/2}}{e^{u^2/2}-1}.
\end{equation}
For \(|u|\ll1\),
\begin{equation}
n_{\rm eff}(u)=2+\frac{u^2}{2}+O(u^4).
\end{equation}
Hence the loss is quadratic for \(|u|\ll1\) and departs from the quadratic form when \(|u|=O(1)\). Since \(u=\sigma_as\), the quadratic-core scale is \(s_{\rm core}=O(1/\sigma_a)\). Defining the core boundary by \(n_{\rm eff}=2.5\) gives \(|u_{\rm core}|\simeq0.964\), or
\begin{equation}
|s_{\rm core}|\simeq\frac{0.964}{\sigma_a}.
\end{equation}
The same derivation applies to confident coordinate-wise BCE because its signed-margin approximation is the same exponential sum over sample--coordinate pairs.
\end{proof}

\section{Full-model experiments}\label{full-model-experiments}

\subsection{Experimental setup for the six models}\label{experimental-setup-for-the-six-models}

All six beta-plane scans use a fixed random seed of 42 and AdamW. The beta values are varied across the scan and are therefore omitted from the setup below. The training horizon is adapted for each task: each run is inspected at successive stages, and unresolved points are resumed at a longer horizon when necessary. The maximum horizon reported here is the largest extension used or permitted for the corresponding final panel, rather than a fixed number of steps applied to every beta setting. Parameter counts refer to trainable model parameters and exclude optimizer state.

\subsubsection{Modular division}\label{a.1.1-modular-division}

\begin{experimentsetup}
\noindent\textbf{Model.} The modular-division model is a one-layer causal Transformer with vocabulary size 54, model dimension \(d_{\rm model}=128\), four attention heads, head dimension \(d_{\rm head}=32\), and an MLP dimension of 512. The sequence length is three, corresponding to the prompt \([x,y,=]\). The model uses ReLU activations and no LayerNorm. The input embedding, positional embedding, attention, MLP, and output unembedding together contain 211,456 trainable parameters.

\noindent\textbf{Dataset.} The task is \(x y^{-1}\bmod 53\), with \(x\in\{0,\ldots,52\}\) and \(y\in\{1,\ldots,52\}\). The complete dataset contains 2,756 examples, which are randomly split into 1,378 training and 1,378 test examples using a 50\% training fraction.

\noindent\textbf{Optimizer.} Training uses cross-entropy loss, AdamW with learning rate \(10^{-3}\), weight decay 0.5, and Adam epsilon \(10^{-8}\). The training loader uses batch size 512, and the learning rate is linearly warmed up during the first 100 optimizer steps.

\noindent\textbf{Training steps.} The final scan starts from 8,000 optimizer steps and adaptively extends unresolved boundary candidates, with a maximum horizon of 30,000 steps.
\end{experimentsetup}

\subsubsection{Modular addition}\label{a.1.2-modular-addition}

\begin{experimentsetup}
\noindent\textbf{Model.} The modular-addition model uses the same one-layer Transformer architecture as modular division: vocabulary size 54, \(d_{\rm model}=128\), four heads with \(d_{\rm head}=32\), MLP dimension 512, sequence length three, ReLU activations, no LayerNorm, and 211,456 trainable parameters.

\noindent\textbf{Dataset.} The task is \((x+y)\bmod 53\), with \(x,y\in\{0,\ldots,52\}\). The complete dataset contains 2,809 examples, randomly split into 1,404 training and 1,405 test examples using a 50\% training fraction.

\noindent\textbf{Optimizer.} Training uses cross-entropy loss, AdamW with learning rate \(10^{-3}\), weight decay 0.5, and Adam epsilon \(10^{-8}\). The batch size is 512, with the same linear 100-step learning-rate warmup as modular division.

\noindent\textbf{Training steps.} The final scan starts from 8,000 optimizer steps and adaptively extends unresolved boundary candidates up to a maximum of 30,000 steps.
\end{experimentsetup}

\subsubsection{Tiny Shakespeare}\label{a.1.3-tiny-shakespeare}

\begin{experimentsetup}
\noindent\textbf{Model.} The Tiny Shakespeare model is a compact causal character Transformer. It has a 65-character vocabulary, model dimension 48, four attention heads, two Transformer blocks, and an MLP dimension of 192 in each block. The character and positional embeddings are 48-dimensional, the positional context length is 64, and the output classifier has no bias. The model uses GELU activations, zero attention dropout, and no LayerNorm. It has 65,472 trainable parameters.

\noindent\textbf{Dataset.} The data are drawn from the first 90\% of \texttt{data/tinyshakespeare/input.txt}, giving a training corpus of 1,003,854 characters. Each run uses the same fixed random set of 448 character windows, each with context length 64. Only the final eight target positions in each window contribute to the loss, while the preceding 56 target positions are masked. The loss is cross-entropy evaluated in float64.

\noindent\textbf{Optimizer.} Training uses full-batch AdamW with learning rate \(10^{-3}\), weight decay 0.1, and Adam epsilon \(10^{-8}\), with no learning-rate scheduler.

\noindent\textbf{Training steps.} The base trajectory is 12,000 optimizer steps. Unresolved points are adaptively extended through 30,000 and 50,000 steps, with selected points extended to a maximum of 100,000 steps.
\end{experimentsetup}

\subsubsection{CIFAR-10 MLP}\label{a.1.4-cifar-10-mlp}

\begin{experimentsetup}
\noindent\textbf{Model.} The CIFAR-10 MLP receives a \(3\times32\times32\) image, flattened to 3,072 input features. It contains five hidden linear layers of width 512, with a ReLU after each hidden layer, followed by a linear 512-to-10 classifier. The model has 2,629,130 trainable parameters and does not use dropout or normalization layers.

\noindent\textbf{Dataset.} For each run, 200 images are randomly selected from the CIFAR-10 training set and 200 images from the CIFAR-10 test set, using seed 42. The images are converted with \texttt{ToTensor()} only, with no data augmentation or additional normalization.

\noindent\textbf{Optimizer.} Optimization uses the 200 training images in full batch with cross-entropy loss and AdamW, using learning rate \(10^{-3}\), weight decay 0.5, and Adam epsilon \(10^{-8}\).

\noindent\textbf{Training steps.} The final scan starts from 4,000 optimizer updates per run and adaptively extends unresolved points to a maximum of 16,000 updates.
\end{experimentsetup}

\subsubsection{CIFAR-10 CNN}\label{a.1.5-cifar-10-cnn}

\begin{experimentsetup}
\noindent\textbf{Model.} The CIFAR-10 CNN is a VGG11-style convolutional network adapted to \(32\times32\) inputs. Its eight convolutional layers have channel widths \(3\to64\to128\to256\to256\to512\to512\to512\to512\), with ReLU activations and five \(2\times2\) max-pooling operations. The classifier maps the final 512 features through two 512-unit linear layers to 10 classes. Each of the two intermediate classifier layers is followed by ReLU and dropout with probability 0.5. The model has 9,750,922 trainable parameters.

\noindent\textbf{Dataset.} The data protocol is the same as for the CIFAR-10 MLP: 200 randomly selected CIFAR-10 training images and 200 randomly selected test images, \texttt{ToTensor()} only, and no augmentation or additional normalization.

\noindent\textbf{Optimizer.} Optimization is full-batch cross-entropy training with AdamW, learning rate \(10^{-3}\), weight decay 0.5, and Adam epsilon \(10^{-8}\).

\noindent\textbf{Training steps.} The base scan uses 8,000 optimizer updates per run. Unresolved points are adaptively extended to 16,000 or 30,000 updates, with a maximum horizon of 30,000 updates.
\end{experimentsetup}

\subsubsection{MNIST autoencoder}\label{a.1.6-mnist-autoencoder}

\begin{experimentsetup}
\noindent\textbf{Model.} The MNIST model is a compact binary convolutional autoencoder. The encoder consists of a \(1\to8\) convolution and an \(8\to16\) convolution, both with kernel size \(3\), stride 2, and ReLU activations, reducing a \(28\times28\) image to a \(16\times7\times7\) representation. A linear layer maps this representation to a 32-dimensional latent vector, and a second linear layer maps it back to \(16\times7\times7\). The decoder uses a \(16\to8\) transposed convolution followed by ReLU and an \(8\to1\) transposed convolution that produces the output logits. Both transposed convolutions use kernel size 4 and stride 2. The autoencoder has 54,425 trainable parameters.

\noindent\textbf{Dataset.} Each run uses 64 randomly selected images from the MNIST training split, with seed 42. The grayscale inputs are binarized at threshold 0.5, and the reconstruction target is the binarized input itself.

\noindent\textbf{Optimizer.} Training uses binary cross-entropy with logits and full-batch AdamW with learning rate \(2\times10^{-3}\), weight decay 0.2, and Adam epsilon \(10^{-12}\). No learning-rate scheduler is used.

\noindent\textbf{Training steps.} The initial phase scan checks successive horizons, and unresolved points are adaptively resumed through a maximum of 200,000 optimizer steps.
\end{experimentsetup}

\subsection{Period-detection algorithm}\label{a.3-period-detection-algorithm}

We estimate the macroscopic period of loss oscillations from loss traces sampled once per optimizer step. We use the training loss for modular arithmetic, Tiny Shakespeare, and MNIST reconstruction, and the test loss for the two CIFAR-10 models. For a trajectory of \(S\) steps, we discard the first \(20\%\) to reduce the influence of the initial transient and analyze the remaining segment. Segments containing fewer than 100 samples are not assigned a period.

To suppress rapid fluctuations while accommodating the large dynamic range of the loss, we first transform the signal as \(u_t=\log_{10}(L_t+10^{-12})\) and smooth it with a third-order Butterworth low-pass filter applied forward and backward. The cutoff frequency is \(0.05\) cycles per step for all models. We then restore the linear loss scale, \(\widetilde L_t=10^{\widetilde u_t}\), and subtract its mean. We estimate the power spectrum \(P(f)\) of this signal using Welch's method.

We identify local spectral peaks and retain only those satisfying both

\begin{equation}
\operatorname{prominence}(f_k)\geq0.02\max_f P(f),
\qquad f_k>\frac{1}{T_{\max}},
\qquad T_{\max}=S.
\end{equation}

Here, prominence measures how strongly a spectral peak stands out from its surrounding spectral background, and \(T_{\max}\) is the length of the complete trajectory before transient removal. Among the accepted peaks, we select the one with the largest spectral power and define the detected period as

\begin{equation}
f_* = \underset{f_k\in\mathcal C}{\arg\max}\,P(f_k),
\qquad T=\frac{1}{f_*},
\end{equation}

where \(\mathcal C\) denotes the set of accepted peaks. The resulting \(T\) is measured in optimizer steps. If no peak satisfies both criteria, the trajectory is marked as having no detected period within the observation window. Unresolved scan points are adaptively extended as described in Appendix~\ref{experimental-setup-for-the-six-models}, and the period is re-estimated from the longer trace. For all models except the CIFAR-10 MLP, we additionally require the minimum training loss in the final \(10\%\) of the complete trajectory to be at most \(0.05\). We exempt the CIFAR-10 MLP because dead ReLU units can prevent its training loss from reaching this threshold even when the trajectory remains relevant to the phase-boundary analysis.

\section{One-dimensional toy-model controls}\label{one-dimensional-toy-model-controls}

\subsection{Sensitivity of the Adam boundary for \texorpdfstring{\(L(x)=kx^2/2\)}{the quadratic loss}}\label{sensitivity-of-adam-on-lxkx22}

To test the sensitivity of the quadratic toy-model boundary to non-beta Adam parameters, we use a common grid spanning \(\delta=1-\beta_1\in[10^{-2},10^{-1}]\) and \(\gamma=1-\beta_2\in[5\times10^{-4},10^{-1}]\). Each control changes one of \(k\), \(\eta\), or \(\epsilon\) relative to the baseline in Figure~\ref{fig:quadratic-controls}(a), while holding the other two fixed. As in the main quadratic experiment, each beta pair is trained for 250,000 updates, the first 125,000 updates are discarded, and a period is assigned when at least three upward crossings of the local stability threshold are detected. We fit the resulting boundary to \(\gamma=C\delta^p\). Figure~\ref{fig:quadratic-controls} shows the phase diagrams, and Table~\ref{tab:quadratic-controls} reports the fits.

\begin{figure}[h]
\centering
\includegraphics[width=\linewidth,keepaspectratio]{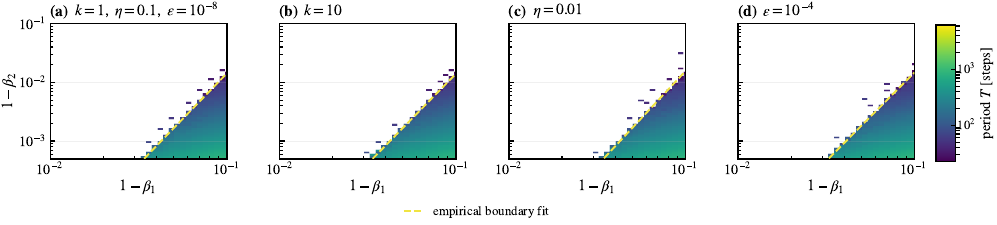}
\caption{Sensitivity of the quadratic-loss boundary to \(k\), \(\eta\), and \(\epsilon\). The phase diagrams show \(L(x)=kx^2/2\) in the \((1-\beta_1,1-\beta_2)\) plane, with dashed fits \(\gamma=C\delta^p\).}
\label{fig:quadratic-controls}
\end{figure}

\begin{table}[h]
\caption{Power-law fits for the quadratic-loss parameter controls in Figure~\ref{fig:quadratic-controls}.}
\label{tab:quadratic-controls}
\centering
\begin{tabular}{@{}llrrr@{}}
\toprule
Panel & Parameters & Valid points & \(p\) & \(C\) \\
\midrule
(a) & \(k=1\), \(\eta=0.1\), \(\epsilon=10^{-8}\) & 351 & 3.095 & 17.70 \\
(b) & \(k=10\), \(\eta=0.1\), \(\epsilon=10^{-8}\) & 353 & 3.064 & 16.30 \\
(c) & \(k=1\), \(\eta=0.01\), \(\epsilon=10^{-8}\) & 351 & 3.210 & 25.00 \\
(d) & \(k=1\), \(\eta=0.1\), \(\epsilon=10^{-4}\) & 365 & 3.009 & 15.02 \\
\bottomrule
\end{tabular}
\end{table}

Changing \(k\) by one order of magnitude leaves the fitted boundary nearly unchanged, and changing \(\eta\) preserves the approximately cubic scaling while producing only a moderate shift in its prefactor. The robustness to \(\epsilon\) is conditional: \(\epsilon\) must remain small compared with \(\sqrt{\hat v_t}\) over the part of the trajectory that determines the boundary. If \(\epsilon\) is too large, it dominates the denominator \(\sqrt{\hat v_t}+\epsilon\) and suppresses the late-time decay of the effective second-moment scale, so the expected boundary scaling is no longer obtained.

\FloatBarrier

\subsection{Effect of weight decay in the one-dimensional toy model}\label{effect-of-weight-decay-on-one-dimensional-toy-model}

Weight decay contributes the parameter update \(\Delta\theta_{\rm wd}=-\eta\lambda_{\rm wd}\theta\). In a neural network, the parameter origin \(\theta=0\) and a local minimum of the task loss are generally distinct. Moreover, the weight-decay update and the Adam-preconditioned gradient direction tend to have small overlaps in high dimensions~\cite{cai2013angles}. At the modular-division pre-spike point used in Figures~\ref{fig:overview}(c) and~\ref{fig:core-wall-rollover}(a), their signed Euclidean projection is
\begin{equation}
\left\langle \Delta\theta_{\rm wd},\frac{\hat d_t}{\lVert\hat d_t\rVert_2}\right\rangle=+9.450\times10^{-3}.
\end{equation}

This geometry cannot be represented by the centered one-dimensional loss \(L(x)=k|x|^n/n\): in one dimension, weight decay and the loss gradient can only be parallel or antiparallel, and decay toward the origin also points toward the loss minimum at \(x=0\). We therefore shift the minimum and use
\begin{equation}
L(x)=\frac{k}{n}|x-x_*|^n,
\qquad x_*=1.
\end{equation}

We choose the small value \(\lambda_{\rm wd}=10^{-3}\) so that the projected decay term perturbs rather than dominates the Adam dynamics. We set \(x_0=2\), preserving the unit initial displacement \(x_0-x_*=1\), and otherwise retain the main toy-model settings: \(k=1\), \(\eta=0.1\), \(\epsilon=10^{-30}\), 250,000 updates, and a 125,000-update burn-in.

\begin{figure}[h]
\centering
\includegraphics[width=\linewidth,keepaspectratio]{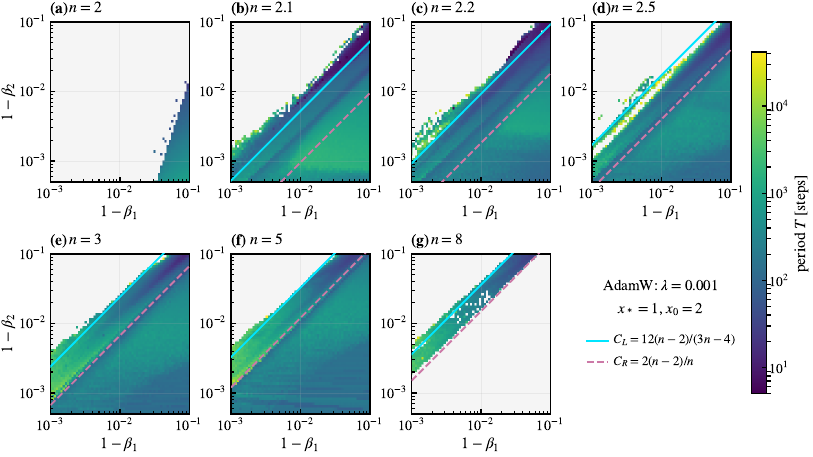}
\caption{Effect of weak decoupled weight decay on the shifted one-dimensional losses \(L(x)=|x-1|^n/n\). AdamW uses \(\eta=0.1\), \(\lambda_{\rm wd}=10^{-3}\), \(\epsilon=10^{-30}\), and \(x_0=2\). Each beta pair is run for 250,000 updates, with the first 125,000 discarded. Color indicates the period estimated from repeated upward crossings of the local stability threshold. The solid cyan and dashed pink lines show \(\gamma=12(n-2)\delta/(3n-4)\) and \(\gamma=2(n-2)\delta/n\), respectively. Relative to Figure~\ref{fig:toy-boundaries}, the left boundary shifts slightly upward, periodic trajectories extend below the reference right boundary, and the right edge of the wedge is unresolved.}
\label{fig:toy-boundaries-weight-decay}
\end{figure}

The weak shifted decay therefore preserves the approximately unit-slope left boundary while modestly enlarging the periodic region on that side. However, the lower/right boundary that closes the no-decay wedge disappears over the scanned range. Thus, the near-linear left-boundary scaling is robust to this weak projected weight-decay force, whereas the two-sided wedge is not.

\FloatBarrier

\section{Finite-scale landscape analysis}\label{finite-scale-landscape-analysis}

\subsection{Locating the pre-spike points and progressive sharpening at finite scales}\label{a.4-locating-the-pre-spike-points}

We define a temporal valley as a low-loss segment between macroscopic high-loss excursions. After discarding the initial training transient, the loss trace is partitioned into successive valleys using the task's loss scale: a valley must enter the low-loss regime and must subsequently be closed by a new high-loss excursion. The last valley in a finite trace is marked as right-censored if no closing excursion has yet been observed and is not used in the exponent fit.

Within each completed valley \(V_j\), the pre-spike point is the single checkpoint
\begin{equation}
t_{{\rm pre},j}=\operatorname*{arg\,min}_{t\in V_j} L_t.
\end{equation}
Thus each completed valley contributes at most one event, preventing long valleys or densely saved portions of a trace from receiving disproportionate weight. The pre-spike point is selected solely from the temporal loss trace. No directional-slope or landscape-shape criterion is used to move the anchor to a nearby checkpoint. When the exact pre-spike point was not retained as a checkpoint during the original run, deterministic replay is used to reconstruct the model and optimizer state at \(t_{{\rm pre},j}\).

Across the examined models and \(\beta\) settings, the selected points occur immediately before a spike. Figure~\ref{fig:loss-traces} illustrates the selection for three modular-division trajectories.

\begin{figure}[h]
\centering
\includegraphics[width=\linewidth,keepaspectratio]{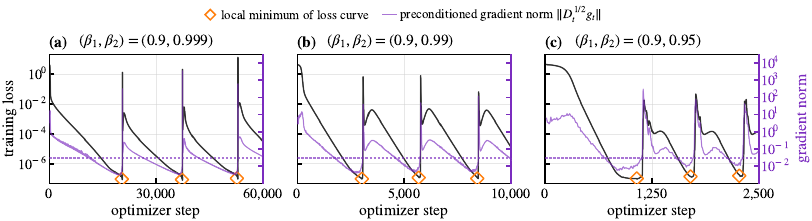}
\caption{Training-loss traces for the modular-division Transformer at \((\beta_1,\beta_2)=(0.9,0.999)\), \((0.9,0.99)\), and \((0.9,0.95)\), from left to right. Diamonds mark the local minima of the loss curve, which are used directly as the pre-spike points.}
\label{fig:loss-traces}
\end{figure}

The pre-spike point \(t_{\rm pre}\) and the spatial minimum \(s_*\) play different roles. The former selects the pre-spike model state. After freezing that state and the Adam preconditioner, we define the full-model directional slice
\begin{equation}
\phi_t(s)=L(\theta_t+s\hat d_t),
\qquad
\hat d_t=\frac{D_t g_t}{\lVert D_t^{1/2} g_t\rVert_2},
\end{equation}
and refine its one-dimensional minimum \(s_*\). The direction combines the current gradient with the second-moment preconditioner and captures the geometry relevant to the Adam update. The exponent fit is centered at \(s_*\), so the procedure does not assume that the pre-spike point itself lies exactly at a stationary point of the frozen directional slice. To compare the slice with the dynamics, we freeze \(\hat d_t\) at step \(t\). For the ten updates ending at \(t\), let \(\Delta\theta_j=\theta_j-\theta_{j-1}\). We calculate their projected slice coordinates and the typical update scale as
\begin{equation}
s_j^{\rm proj}=\frac{\Delta\theta_j^\top\hat d_t}{\hat d_t^\top\hat d_t},
\qquad
\overline{\Delta s}=\underset{j=t-9,\ldots,t}{\operatorname{median}}\left|s_j^{\rm proj}\right|.
\end{equation}
The denominator expresses each projected update in the same \(s\) coordinate used by \(\phi_t(s)\), even though \(\hat d_t\) is not Euclidean-normalized.

We classify each recentered slice according to Definition~\ref{def:core-wall-landscape}. We use the modular-division Transformer at \((\beta_1,\beta_2)=(0.9,0.999)\) as a representative example.

\begin{figure}[h]
\centering
\includegraphics[width=\linewidth,keepaspectratio]{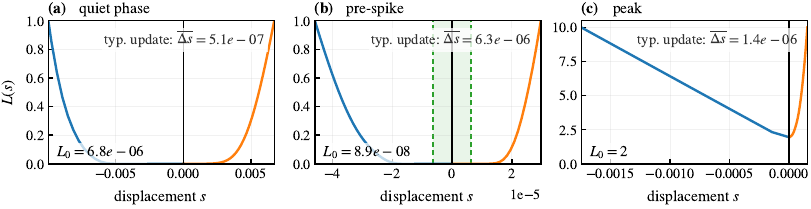}
\caption{Evolution of the directional loss slice through a modular-division spike at \((\beta_1,\beta_2)=(0.9,0.999)\). The three panels show a quiet point (step \(28{,}600\), \(L_0\simeq6.8\times10^{-6}\)), the pre-spike point (step \(37{,}210\), \(L_0\simeq8.9\times10^{-8}\)), and the spike peak (step \(37{,}220\), \(L_0\simeq2.0\)), respectively. The green band and dashed edges indicate the magnitude of a typical Adam-preconditioned update: \(\overline{\Delta s}\simeq5.1\times10^{-7}\), \(6.3\times10^{-6}\), and \(1.4\times10^{-6}\) in the three panels. As the system approaches the pre-spike state, the flat core contracts to the update scale. \textbf{This phenomenon can be viewed as progressive sharpening at finite scales.} At the peak, the directional wall is encountered and the slice becomes strongly nonquadratic.}
\label{fig:three-stages}
\end{figure}

We define the core--wall transition as the first displacement at which
\begin{equation}
n_{\rm eff}(s)=\frac{\mathrm d\log\Delta L}{\mathrm d\log|s-s_*|}
\end{equation}
reaches \(2.5\). During the quiet phase in Figure~\ref{fig:three-stages}, the left and right transition distances are approximately \(7.0\times10^{-4}\) and \(8.0\times10^{-4}\), more than three orders of magnitude larger than \(\overline{\Delta s}\simeq5.1\times10^{-7}\). At the pre-spike point, these distances contract to \(2.16\times10^{-6}\) and \(3.55\times10^{-6}\), while \(\overline{\Delta s}\) grows to \(6.3\times10^{-6}\). \textbf{A typical update therefore crosses the quadratic core and reaches the superquadratic wall.} At the spike peak, the reference loss is approximately \(2.0\), and the strongly asymmetric slice no longer exhibits a core--wall landscape.

\FloatBarrier

\subsection{Fitting and aggregating the wall exponent}\label{a.5-fitting-and-aggregating-the-wall-exponent}

Every model--task setting is evaluated at the same six reference beta pairs,
\begin{equation}
(\beta_1,\beta_2)\in\{(0.9,0.9),(0.9,0.95),(0.9,0.99),(0.9,0.999),(0.99,0.99),(0.99,0.999)\}.
\end{equation}
For each usable pre-spike slice, define the displacement and excess loss on its left and right branches by
\begin{equation}
r=|s-s_*|,
\qquad
\Delta L=\phi_t(s)-\phi_t(s_*),
\end{equation}
and transform the positive, resolved samples to \(u=\log r\) and \(y=\log\Delta L\). The two branches are retained separately so that basin asymmetry is absorbed by branch-specific parameters rather than by averaging the profiles before fitting.

Each event-side branch is first fit with a continuous three-segment line,
\begin{equation}
y(u)=c+n_{\rm core}(u-b_1)+(n_{\rm wall}-n_{\rm core})[u-b_1]_+ +(n_{\rm roll}-n_{\rm wall})[u-b_2]_+,
\qquad [z]_+=\max(z,0).
\end{equation}
The breakpoints \(b_1<b_2\) separate the quadratic core, steep wall, and outer rollover. Candidate breakpoints are enumerated, with at least five samples required in each regime, and the pair with the smallest total squared residual in log space is selected subject to
\begin{equation}
n_{\rm wall}>n_{\rm core},
\qquad
n_{\rm wall}>n_{\rm roll}.
\end{equation}
These inequalities operationally identify the wall as the middle steepening region. They also prevent the flatter outer rollover from being folded into the phase-relevant wall exponent.

After determining \(b_1\) and \(b_2\) separately for every event and side, all branches within one beta setting are refit jointly using only samples through \(b_2\):
\begin{equation}
y_{e,\pm}(u)=c_{e,\pm}+n_{\rm core}\min(u-b_{1,e,\pm},0)+n_{\rm wall}\max(u-b_{1,e,\pm},0),\qquad u\leq b_{2,e,\pm}.
\end{equation}
The intercept and breakpoints remain specific to each event-side branch, while \(n_{\rm core}\) and \(n_{\rm wall}\) are shared. This produces one beta-level exponent \(n_{{\rm wall},k}\) from all usable pre-spike events at the \(k\)-th beta pair. Uncertainty is estimated with 2,000 event-level bootstrap resamples, with the left and right branches of each event always resampled together.

Finally, the exponent shown in the main-text table is the unweighted arithmetic mean of the six beta-level estimates,
\begin{equation}
n_{\rm wall}^{\rm(task)}=\frac{1}{6}\sum_{k=1}^{6}n_{{\rm wall},k}.
\end{equation}
Each beta pair therefore receives the same weight even when the number of usable completed valleys differs. The boundary coefficient is computed only after this averaging step,
\begin{equation}
C_L=\frac{12\bigl(n_{\rm wall}^{\rm(task)}-2\bigr)}{3n_{\rm wall}^{\rm(task)}-4},
\end{equation}
rather than by averaging six separately transformed coefficients.

\end{document}